\documentclass[11pt]{article}
\usepackage{ifpdf}
\usepackage{geometry} 
\usepackage[parfill]{parskip}    
\ifpdf
\usepackage[pdftex]{graphicx}
\else
\usepackage[dvips]{graphicx}
\fi
\usepackage{amsmath, amsthm}
\usepackage{amsfonts}
\usepackage{amssymb}
\usepackage{epstopdf}
\usepackage{subfigure}
\usepackage{fancyhdr}
\usepackage{algorithms/algorithmic}
\usepackage{algorithms/algorithm}
\usepackage{array}
\usepackage{url}
\usepackage{comment}
\usepackage{natbib}
\usepackage{empheq}

\newtheorem{thm}{Theorem}[section]

\newtheorem{prop}[thm]{Proposition}
\theoremstyle{definition}

\numberwithin{equation}{section}

\usepackage[usenames,dvipsnames]{color}
\usepackage{bbding}

\ifpdf
\usepackage[pdftex]{hyperref}
\else
\usepackage[dvips]{hyperref}
\fi
\hypersetup{
    bookmarks=true,         
    unicode=false,          
    pdftoolbar=true,        
    pdfmenubar=true,        
    pdffitwindow=false,     
    pdfstartview={FitH},    
    pdftitle={BSVM},    
    pdfauthor={Gautam V. Pendse},     
    pdfsubject={Banded SVM},   
    pdfcreator={Gautam V. Pendse},   
    pdfproducer={Gautam V. Pendse}, 
    pdfkeywords={SVM} {support vector machine} {BSVM}, 
    pdfnewwindow=true,      
    colorlinks=true,       
    linkcolor=red,          
    citecolor=blue,        
    filecolor=magenta,      
    urlcolor=cyan           
}

\definecolor{orange}{rgb}{1,0.5,0}

\newcommand{\vect}[1]{\boldsymbol{#1}}
\newcommand{\matr}[1]{\boldsymbol{#1}}

\newcommand{\HRule}{\rule{\linewidth}{0.5mm}}

\title{BSVM: A Banded Suport Vector Machine}

\author{Gautam V. Pendse\thanks{To whom correspondence should be addressed. e-mail: gpendse@mclean.harvard.edu}$^{\,\,\,1}$ \mbox{}\\ \\ \\ \\ 
$^1$ P.A.I.N Group, Imaging and Analysis Group (IMAG), McLean Hospital, Harvard Medical School}

\begin{document}

\begin{titlepage}
\begin{center}

\includegraphics[width=0.15\textwidth]{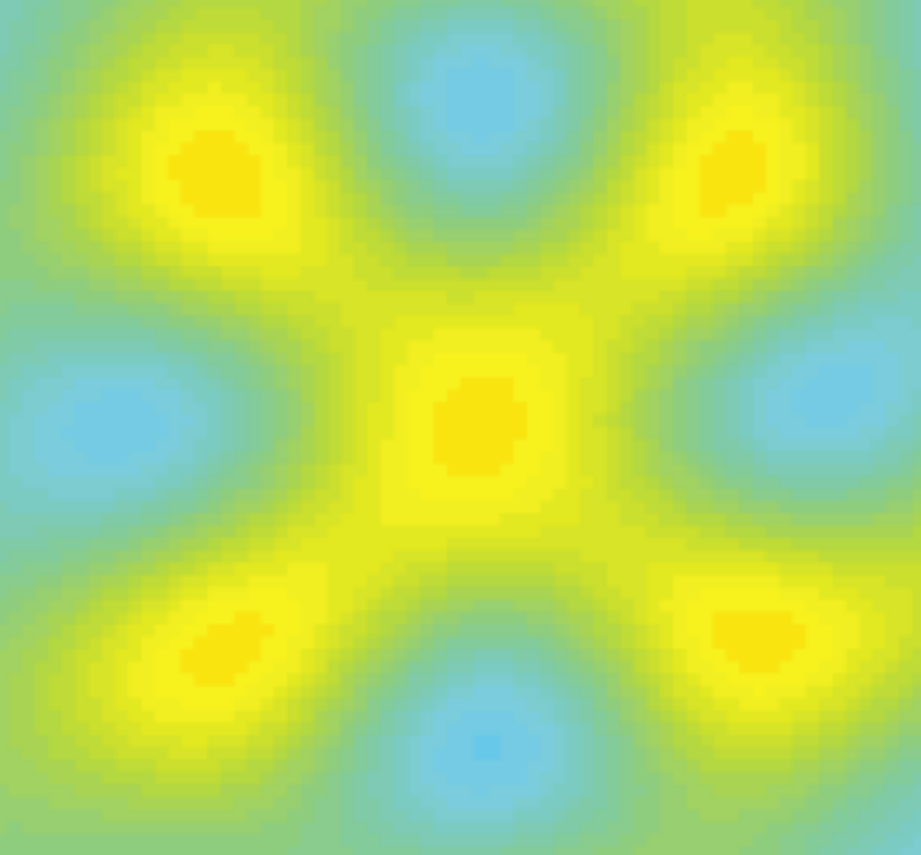}\\[1cm]    

\HRule\\[0.4cm]
\textbf{ \Huge BSVM}
\HRule\\[1cm]
 {\textsc{\LARGE A Banded Support Vector Machine}} \\[2cm]

\begin{minipage}{1 \textwidth}
\begin{center}

\textit{ \Large{Author:} }\\[1cm]

\textbf{\Large{ Gautam V. Pendse }}\\[0.5cm]
{\large \verb+gpendse@mclean.harvard.edu+}\\[0.5cm]
\large{P.A.I.N Group, Brain Imaging Center}\\
\large{McLean Hospital, Harvard Medical School}\\

\end{center}
\end{minipage}\\[1cm]

\vfill

\textsc{\large \today}

\end{center}
\end{titlepage}


\newpage 

\begingroup
\hypersetup{linkcolor=red}
\tableofcontents
\listoffigures
\endgroup

\section*{Abstract}
We describe a novel binary classification technique called Banded SVM (B-SVM). In the standard C-SVM formulation of \citet{Cortes:1995}, the decision rule is encouraged to lie in the interval $[1, \infty]$. The new B-SVM objective function contains a penalty term that encourages the decision rule to lie in a user specified range $[\rho_1, \rho_2]$. In addition to the standard set of support vectors (SVs) near the class boundaries, B-SVM results in a second set of SVs in the interior of each class.


\section*{Notation}
\begin{itemize}
\item[\PencilRight] Scalars and functions will be denoted in a non-bold font (e.g., $\beta_0, C, g$). Vectors and vector functions will be denoted in a bold font using lower case letters (e.g., $\vect{x}, \vect{\beta}, \vect{h}$). Matrices will be denoted in bold font using upper case letters (e.g., $\matr{B}, \matr{H}$). The transpose of a matrix $\matr{A}$ will be denoted by $\matr{A^T}$ and its inverse will be denoted by $\matr{A^{-1}}$. $\matr{I_p}$ will denote the $p \times p$ identity matrix and $\mathbf{0}$ will denote a vector or matrix of all zeros whose size should be clear from context.

\item[\PencilRight] $|x|$ will denote the absolute value of $x$ and $\mathcal{I}(x > a)$ is an indicator function that returns $1$ if $x > a$ and $0$ otherwise.

\item[\PencilRight] The $j$th component of vector $\vect{t}$ will be denoted by $t_{j}$. The element $(i,j)$ of matrix $\matr{G}$ will be denoted by $G(i,j)$ or $G_{ij}$. The 2-norm of a $p \times 1$ vector $\vect{x}$ will be denoted by $|| \vect{x} ||_2 = +\sqrt{ \sum_{i = 1}^p x_i^2 }$. Probability distribution of a random vector $\vect{x}$ will be denoted by $\mathbf{P}_{\vect{x}}(\vect{x})$. $\mathbf{E} \left[ f(\vect{s}, \vect{\eta} ) \right]$ denotes the expectation of $f(\vect{s}, \vect{\eta})$ with respect to both random variables $\vect{s}$ and $\vect{\eta}$.

\end{itemize}

\section{Introduction}

We consider the standard binary classification problem. Suppose $y_i$ is the class membership label ($+1$ for class $+1$ and $-1$ for class $-1$) associated with a feature vector $\vect{x_i}$. Given $n$ such $(\vect{x_i}, y_i)$ pairs, we would like to learn a linear decision rule $g(\vect{x})$ that can be used to accurately predict the class label $y$ associated with feature vector $\vect{x}$.

In C-SVM \citep{Vapnik:1963, Boser:1992, Cortes:1995}, one can think of the linear decision rule $g$ as a means of measuring membership in a particular class. Given a feature vector $\vect{x}$, C-SVM encourages the function $g(\vect{x})$ to be positive if $\vect{x} \in$ class $+1$ and negative if $\vect{x} \in$ class $-1$. 

We motivate the development of B-SVM in the following way. Suppose that vector $\vect{x}$ comes from an arbitrary probability distribution $\mathbf{P}_{\vect{x}}(\vect{x})$ with mean $\mathbf{E}[\vect{x}] = \vect{\mu}$ and finite co-variance $\mbox{Cov}[\vect{x}] = \matr{\Sigma}$. Consider the linear decision rule $g(\vect{x}) = \vect{\beta}^T \vect{x} + \beta_0$. It is easy to see that $g(\vect{x})$ has mean $\mathbf{E}[g(\vect{x})] = \vect{\beta}^T \vect{\mu} + \beta_0$ and covariance $\mbox{Cov}[g(\vect{x})] = \vect{\beta}^T \matr{\Sigma} \vect{\beta}$. By Chebyshev's inequality, there exists a high probability band around $\mathbf{E}[g(\vect{x})]$ where $g(\vect{x})$ is expected to lie when $\vect{x}$ comes from $\mathbf{P}_{\vect{x}}(\vect{x})$.

Hence, for every probability distribution of vectors $\vect{x}$ from class $+1$ and class $-1$ with finite co-variance, $g(\vect{x})$ is expected to lie in a certain high probability band. In B-SVM, we choose $g(\vect{x})$ to encourage:
\begin{itemize}
\item[\PencilRight] $y \, g(\vect{x}) > 0 \,\,\,$ \HandLeft $\,$ same condition as C-SVM
\item[\PencilRight] $y \, g(\vect{x}) \in \mbox{certain high probability \textbf{\textit{band}}} \,\,\,$ \HandLeft $\,$ new B-SVM condition
\end{itemize}

Both of the above conditions can be satisfied if we encourage:

\begin{equation}\label{eq0}
\boxed{y \, g(\vect{x}) \in [\rho_1, \rho_2] \mbox{ with } \rho_2 > \rho_1 > 0}
\end{equation}

Since non-linear decision rules in C-SVM are simply linear decision rules operating in a high dimensional space via the kernel trick \citep{Boser:1992}, the B-SVM band formation argument holds for non-linear decision rules as well.

\section{Problem setup}
As per standard SVM terminology, assume that we are given $n$ data-label pairs $(\vect{x_i}, y_i)$ where $\vect{x_i}$ are $m \times 1$ vectors and the data labels $y_i \in \{-1, 1\}$. First, we consider only the linear case and afterwards transform to the general case via the kernel trick. Let $m \times 1$ vector $\vect{\beta}$ and scalar $\beta_0$ be parameters of a linear decision rule $g(\vect{x}) = \vect{\beta}^T \vect{x} + \beta_0 = 0$ separating class $+1$ and $-1$ such that $g(\vect{x})  > 0$ if $\vect{x}$ belongs to class $+1$ and vice versa. 

\subsection{C-SVM objective function}
The C-SVM objective function \citep{Cortes:1995} to be minimized can be written as:
\begin{equation}\label{1}
f_{CSVM}(\vect{\beta}, \beta_0) = \frac{1}{2} || \vect{\beta} ||^2_2 + C \sum_{i = 1}^n [1 - y_i(\vect{\beta}^T \vect{x_i} + \beta_0)]_{+}
\end{equation}
where $[t]_{+}$ is the positive part of $t$:
\begin{align}\label{1a}
[t]_{+}=\begin{cases}
0& \text{if $t \le 0$},\\
t& \text{if $t > 0$}.
\end{cases}
\end{align}
and $C$ governs the regularity of the solution. The C-SVM objective function penalizes signed decisions $y_i (\vect{\beta}^T \vect{x_i} + \beta_0)$ whenever their value is below 1. This is the only penalty in C-SVM. 

\subsection{B-SVM objective function}
We present below the novel B-SVM objective function that we wish to minimize:
\begin{equation}\label{2}
f_{BSVM}(\vect{\beta}, \beta_0) = \frac{1}{2} ||\vect{\beta}||^2_2 \, + \,  \underbracket[2pt]{ C_1 \sum_{i = 1}^n [\rho_1 - y_i(\vect{\beta}^T \vect{x_i} + \beta_0)]_{+} }_{\mbox{\textcolor{red}{C-SVM like penalty}}} \, + \,  \underbracket[2pt]{ C_2 \sum_{i = 1}^n [y_i(\vect{\beta}^T \vect{x_i} + \beta_0) - \rho_2]_{+} }_{\mbox{\textcolor{red}{novel B-SVM penalty}}}
\end{equation}
where $\rho_2 > \rho_1 > 0$ are margin parameters specified by the user and $C_1$ and $C_2$ are regularization constants. This objective function has two penalty terms:
\begin{itemize}
\item[\PencilRight] The first penalty term is similar to C-SVM. It penalizes signed decisions $y_i (\vect{\beta}^T \vect{x_i} + \beta_0)$ whenever their values are below $\rho_1$ (as opposed to 1 in C-SVM). 

\item [\PencilRight] The second penalty term is novel. It penalizes signed decisions $y_i (\vect{\beta}^T \vect{x_i} + \beta_0)$ when their values are above $\rho_2$. 
\end{itemize}
The net effect of these penalty terms is to encourage $y_i (\vect{\beta}^T \vect{x_i} + \beta_0)$ to lie in the interval $[\rho_1, \rho_2]$. Please see Figure \ref{figure1} for a sketch of the two penalty terms in B-SVM. 

\begin{figure}[htbp]
\begin{center}
\includegraphics[width = 6.0in] {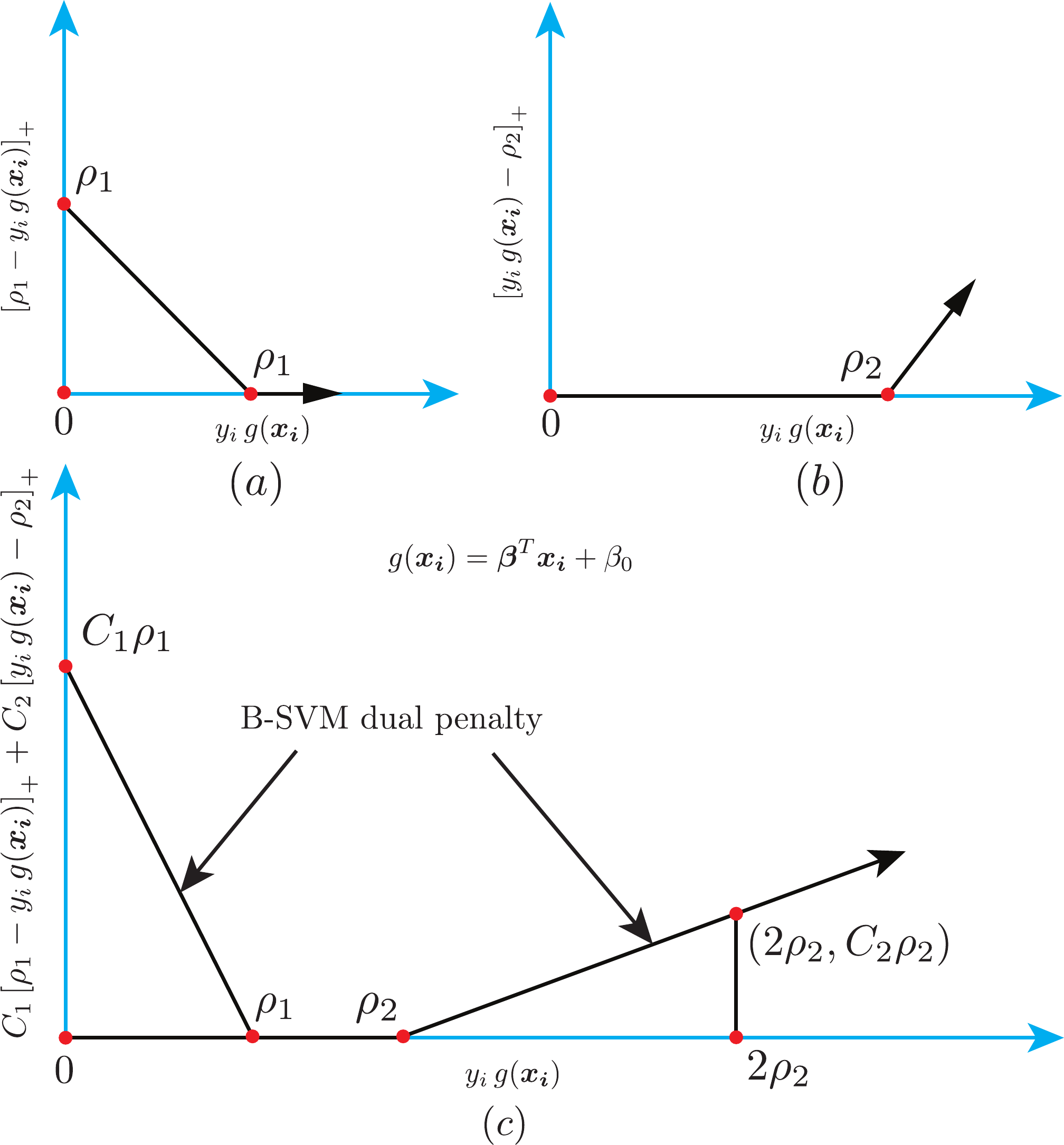}
\caption{(a) Standard C-SVM like penalty function penalizes $y_i (\vect{\beta}^T \vect{x_i} + \beta_0) < \rho_1$. In B-SVM, $\rho_1$ replaces the constant $1$ from C-SVM. (b) Novel B-SVM penalty function. This function penalizes $y_i( \vect{\beta}^T \vect{x_i} + \beta_0) > \rho_2$. (c) Total penalty function for B-SVM. If $y_i (\vect{\beta}^T \vect{x_i} + \beta_0) \in [\rho_1, \rho_2]$ then the total penalty is $0$. Choosing $C_2 < C_1$ will impose a milder penalty for values of $y_i (\vect{\beta}^T \vect{x_i} + \beta_0) > \rho_2$. }
\label{figure1}
\end{center}
\end{figure}

\section{Solving the B-SVM problem}
We derive the B-SVM dual problem in order to maximize a lower bound on the B-SVM primal objective function in equation \ref{2}. This dual problem will be simpler to solve compared to the primal form \ref{2}. We proceed as follows:
\begin{itemize}
\item[\PencilRight] As shown in \ref{3}, the primal problem in \ref{2} can be modified into a \textit{strictly} convex objective function with linear inequality constraints using slack variables. 

\item[\PencilRight] Consequently, \textit{strong duality} holds and the maximum value of the B-SVM dual objective function is equal to the minimum value of the B-SVM primal objective function in \ref{2}. 
\end{itemize}
For more details on convex duality, please see \cite{Nocedal:book}.

\subsection{The B-SVM dual problem}
We introduce slack variables:
\begin{align}\label{eq3a}
\xi_i &= [\rho_1 - y_i(\vect{\beta}^T \vect{x_i} + \beta_0)]_{+} \\
\eta_i &= [y_i(\vect{\beta}^T \vect{x_i} + \beta_0) - \rho_2]_{+} \nonumber
\end{align} 
into the primal objective function in \ref{2}. The modified optimization problem can be written as:
\begin{align}\label{3}
\min_{\vect{\beta}, \beta_0, \vect{\xi}, \vect{\eta}} f_{BSVM}(\vect{\beta}, \beta_0, \vect{\xi}, \vect{\eta}) &=  \frac{1}{2} ||\vect{\beta}||^2_2 + C_1 \sum_{i = 1}^n \xi_i + C_2 \sum_{i = 1}^n \eta_i  \\
\xi_i &\ge 0 \nonumber && \text{Lagrange multiplier $\mu_i$}\\
\eta_i &\ge 0 \nonumber && \text{Lagrange multiplier $\psi_i$}\\
\xi_i &\ge \rho_1 - y_i(\vect{\beta}^T \vect{x_i} + \beta_0) \nonumber && \text{Lagrange multiplier $\alpha_i$}\\
\eta_i &\ge -\rho_2 + y_i(\vect{\beta}^T \vect{x_i} + \beta_0) \nonumber && \text{Lagrange multiplier $\theta_i$}
\end{align} 
After introducing Lagrange multipliers for each inequality constraint as shown in \ref{3}, the Lagrangian function for problem \ref{3} can be written as:
\begin{align}\label{4}
L(\vect{\beta}, \beta_0, \vect{\xi}, \vect{\eta}, \vect{\alpha}, \vect{\theta}, \vect{\mu}, \vect{\psi}) =  \frac{1}{2} || \vect{\beta} ||^2_2 \, + \, &C_1 \sum_{i = 1}^n \xi_i + C_2 \sum_{i = 1}^n \eta_i - \sum_{i = 1}^n \alpha_i \{ \xi_i - \rho_1 + y_i (\vect{\beta}^T \vect{x_i} + \beta_0)\}  \\
&- \sum_{i = 1}^n \theta_i \{ \eta_i + \rho_2 - y_i(\vect{\beta}^T \vect{x_i} + \beta_0) \} - \sum_{i = 1}^n \mu_i \xi_i - \sum_{i = 1}^n \psi_i \eta_i \nonumber
\end{align}
where
\begin{equation}\label{4a}
\alpha_i, \theta_i, \mu_i, \psi_i \ge 0
\end{equation}
Next, we solve for primal variables $\vect{\beta}, \beta_0, \vect{\xi}, \vect{\eta}$ in terms of the dual variables $\vect{\alpha}, \vect{\theta}, \vect{\mu}, \vect{\psi}$ by minimizing $L(\vect{\beta}, \beta_0, \vect{\xi}, \vect{\eta}, \vect{\alpha}, \vect{\theta}, \vect{\mu}, \vect{\psi})$ with respect to the primal variables. Since the Lagrangian in \ref{4} is a convex function of the primal variables, its unique global minimum can be obtained using the first order Karush Kuhn Tucker (KKT) conditions given in \ref{5} - \ref{8}:
\begin{equation}\label{5}
\frac{\partial L}{\partial \vect{\beta}} = \vect{\beta} - \sum_{i = 1}^n \alpha_i y_i \vect{x_i} + \sum_{i = 1}^n \theta_i y_i \vect{x_i} = 0
\end{equation}

\begin{equation}\label{6}
\frac{\partial L}{\partial \beta_0} = -\sum_{i = 1}^n \alpha_i y_i + \sum_{i = 1}^n \theta_i y_i = 0
\end{equation}

\begin{equation}\label{7}
\frac{\partial L}{\partial \xi_l} = C_1 - \alpha_l - \mu_l = 0
\end{equation}

\begin{equation}\label{8}
\frac{\partial L}{\partial \eta_l} = C_2 - \theta_l - \psi_l = 0
\end{equation}

From \ref{5}, the vector $\vect{\beta}$ is given by:
\begin{equation}\label{eq1}
 \vect{\beta} = \sum_{i = 1}^n (\alpha_i - \theta_i) \, y_i \, \vect{x_i} 
 \end{equation}
From \ref{6}, vectors $\vect{\alpha}$ and $\vect{\theta}$ satisfy the equality constraint: 
 \begin{equation}\label{eq2}
 \sum_{i = 1}^n (\alpha_i - \theta_i) \, y_i = 0
 \end{equation}

Combining \ref{7}, \ref{8} and \ref{4a}, the elements of $\vect{\alpha}$ must satisfy: 
\begin{equation}\label{eq3}
0 \le \alpha_i \le C_1
\end{equation}
and elements of $\vect{\theta}$ satisfy: 
\begin{equation}\label{eq4}
0 \le \theta_i \le C_2
\end{equation}
Let $\matr{B}$ be a $n \times n$ matrix with entries:
\begin{equation}\label{eq4a}
B_{ij} = y_i y_j \, \vect{x_i}^T \vect{x_j} 
\end{equation}
 and $\vect{e}_n$ be a $n \times 1$ vector of $n$ ones (in MATLAB notation: $\vect{e}_n$ = \verb+ones(n,1)+). Substituting $\vect{\beta}$ from \ref{eq1} in \ref{4} and noting the constraints \ref{7}, \ref{8} and \ref{eq2}, we get the B-SVM dual problem:
\begin{empheq}[box=\fbox]{align}\label{8a}
\max_{\vect{\alpha}, \vect{\theta}} L_D(\vect{\alpha}, \vect{\theta}) &= \rho_1 \, \vect{e}^T_n \vect{\alpha} - \rho_2 \, \vect{e}^T_n \vect{\theta} -\frac{1}{2} (\vect{\alpha} - \vect{\theta})^T \matr{B} (\vect{\alpha} - \vect{\theta}) \\
&\vect{0} \le \vect{\alpha} \le C_1 \, \vect{e}_n \nonumber \\
&\vect{0} \le \vect{\theta} \le C_2 \, \vect{e}_n \nonumber \\
&(\vect{\alpha} - \vect{\theta})^T \vect{y} = 0 \nonumber
\end{empheq}

If $C_2 = 0$ and $\rho_1 = 1$ then \ref{eq4} implies $\vect{\theta} = \vect{0}$ and hence we recover the standard C-SVM dual problem. 

\subsection{Kernelifying B-SVM}
Let $\vect{h}$ be a non-linear vector function that takes inputs $\vect{x_i}$ into a high dimensional space. Then we recover \textbf{\textit{kernel}} B-SVM by doing linear B-SVM on the data-label pairs $(\vect{h}(\vect{x_i}), y_i)$ instead of the original pairs $(\vect{x_i}, y_i)$. In practice, we do not need $\vect{h}(\vect{x})$ explicitly but only the dot products  through a kernel matrix $\matr{K}$ with elements:
\begin{equation}\label{eq4b} 
K_{ij} = K(\vect{x_i},\vect{x_j}) = \vect{h}(\vect{x_i})^T \vect{h}(\vect{x_j})
\end{equation}
This is the so-called kernel trick. From \ref{eq4a}, elements of matrix $\matr{B}$ for transformed feature vectors $\vect{h}(\vect{x})$ are given by:
\begin{equation}\label{eq4c}
B_{ij} = y_i y_j \, \vect{h}(\vect{x_i})^T  \vect{h}(\vect{x_j})  = y_i y_j \, K_{ij} = y_i y_j \, K(\vect{x_i},\vect{x_j})
\end{equation}
For a new point $\vect{x}$, the decision rule is then given by: 
\begin{equation}\label{eq4ca}
g(\vect{x}) = \vect{\beta}^T \vect{h}(\vect{x}) + \beta_0
\end{equation}
and $\vect{x}$ is classified into class $+1$ if $g(\vect{x}) > 0$  and into class $-1$ if $g(\vect{x}) < 0$. From \ref{eq1}, for the transformed feature vectors $\vect{h}(\vect{x_i})$, we have:
\begin{equation}\label{eq4cb}
\vect{\beta} = \sum_{i = 1}^n (\alpha_i - \theta_i) \, y_i \, \vect{h}(\vect{\vect{x_i}})
\end{equation}
Using the kernel trick, calculation of $g(\vect{x})$ does not need $\vect{h}(\vect{x})$ explicitly as we can write:
\begin{equation}\label{eq4d}
\boxed{
g(\vect{x}) = \vect{\beta}^T \vect{h}(\vect{x}) + \beta_0 =  \sum_{i = 1}^n (\alpha_i - \theta_i) \, y_i \, K(\vect{x_i}, \vect{x}) + \beta_0
}
\end{equation}

\begin{prop}
The B-SVM dual objective function $L_D(\vect{\alpha}, \vect{\theta})$ in \ref{8a} is a concave function of $\vect{\alpha}$ and $\vect{\theta}$.
\end{prop}
\begin{proof}
Since $\matr{B}$ is symmetric, the Hessian of $L_D$ with respect to the vector $(\vect{\alpha}, \vect{\theta})$ is given by:
\begin{equation}
\matr{H} = \begin{pmatrix}
-\matr{B} & \matr{B}\\
\matr{B} & -\matr{B}
\end{pmatrix}
\end{equation}
If $\vect{c}$ and $\vect{d}$ are arbitrary $n \times 1$ vectors,
\begin{align}\label{prop_eq1}
\begin{pmatrix} 
\vect{c}^T & \vect{d}^T
\end{pmatrix} \matr{H} \begin{pmatrix} \vect{c} \\ \vect{d} \end{pmatrix}
= \vect{c}^T (-\matr{B} \vect{c} + \matr{B} \vect{d}) + \vect{d}^T (\matr{B} \vect{c} - \matr{B} \vect{d}) = -(\vect{c} - \vect{d})^T \matr{B} \, (\vect{c} - \vect{d})
\end{align}

From \ref{eq4c}, 
\begin{equation}
(\vect{c} - \vect{d})^T \matr{B} \, (\vect{c} - \vect{d}) = \sum_{i = 1}^n \sum_{j = 1}^n (\vect{c}-\vect{d})_i \{ y_i y_j \matr{K}(\vect{x_i}, \vect{x_j}) \} (\vect{c}-\vect{d})_j = \sum_{i = 1}^n \sum_{j = 1}^n \{ (\vect{c}-\vect{d})_i y_i \} \matr{K}(\vect{x_i}, \vect{x_j}) \{(\vect{c}-\vect{d})_j y_j\}
\end{equation}

If $\odot$ is an element-wise multiplication operator then:
\begin{equation}\label{prop_eq2}
(\vect{c} - \vect{d})^T \matr{B} \, (\vect{c} - \vect{d}) = \{ (\vect{c}-\vect{d}) \odot \vect{y} \}^T \matr{K} \{(\vect{c}-\vect{d}) \odot \vect{y}\} \ge 0
\end{equation}
where the last inequality holds since $\matr{K}$ is a kernel matrix which is positive definite by \ref{eq4b}. Therefore, from \ref{prop_eq1} and \ref{prop_eq2}:
\begin{equation}\label{prop_eq3}
\begin{pmatrix} \vect{c}^T & \vect{d}^T \end{pmatrix} \matr{H} \begin{pmatrix} \vect{c} \\ \vect{d} \end{pmatrix} \le 0
\end{equation}
for all vectors $\vect{c}$ and $\vect{d}$. Thus $L_D(\vect{\alpha},\vect{\theta})$ is a concave function of $(\vect{\alpha}, \vect{\theta})$. \qed
\end{proof} 

It immediately follows that problem \ref{8a} attempts to maximize a concave function under linear constraints and thus has a unique solution \citep{Nocedal:book}.

\subsection{Calculation of dual variables}
Dual variables $\vect{\alpha}$, $\vect{\theta}$, $\vect{\mu}$, $\vect{\psi}$ can be calculated as follows:
\begin{itemize}
\item[\PencilRight] Calculation of $\vect{\alpha}$, $\vect{\theta}$ requires the solution of a concave maximization problem \ref{8a} where the elements of $\matr{B}$ are chosen using a suitable kernel $\matr{K}(\vect{x_i}, \vect{x_j})$. This can be accomplished using an sequential minimal optimization (SMO) type active set technique \citep{Platt1998} or a projected conjugate gradient (PCG) technique \citep{Nocedal:book}. 

\item[\PencilRight] Once $\vect{\alpha}$ and $\vect{\theta}$ are known, equations \ref{7} and \ref{8} give $\vect{\mu} = C_1 \vect{e}_n - \vect{\alpha}$ and $\vect{\psi} = C_2 \vect{e}_n - \vect{\theta}$.
\end{itemize}

\subsection{Calculation of primal variables}

Primal variables $\vect{\beta}$, $\vect{\beta_0}$, $\vect{\xi}$, $\vect{\eta}$ can be calculated as follows:
\begin{itemize}
\item[\PencilRight] $\vect{\beta}$ is given by equation \ref{eq4cb}.

\item[\PencilRight] Calculation of  $\vect{\beta_0}$, $\vect{\xi}$, $\vect{\eta}$ is accomplished by considering the inequality constraints and the KKT \textit{complementarity} constraints for the problem \ref{3}:

\begin{align}\label{9}
&\xi_i \ge 0, \eta_i \ge 0 \\  \nonumber
&\xi_i \ge \rho_1 - y_i \left( \vect{\beta}^T \vect{h}(\vect{x_i}) + \beta_0 \right) \\ \nonumber
&\eta_i \ge -\rho_2 + y_i \left( \vect{\beta}^T \vect{h}(\vect{x_i}) + \beta_0 \right) \\ \nonumber
&\alpha_i \{ \xi_i - \rho_1 + y_i \left( \vect{\beta}^T \vect{h}(\vect{x_i}) + \beta_0 \right)\} = 0  \\ \nonumber
&\theta_i \{ \eta_i + \rho_2 - y_i \left( \vect{\beta}^T \vect{h}(\vect{x_i}) + \beta_0 \right) \}  = 0  \\ \nonumber
&\mu_i \xi_i = (C_1 - \alpha_i) \xi_i  = 0 \\ \nonumber
&\psi_i \eta_i = (C_2 - \theta_i) \eta_i = 0 
\end{align}

Given the positivity constraints \ref{4a} and the bound constraints \ref{eq3} and \ref{eq4}, we consider the following cases:

\begin{itemize}

\item[\HandRight] If $\alpha_i < C_1$ then $\xi_i = 0$ and similarly if $\theta_i < C_2$ then $\eta_i = 0$. 

\item[\HandRight] If $0 < \alpha_i < C_1$ then we have $\xi_i = 0$ and $\{ \xi_i - \rho_1 + y_i (\beta^T x_i + \beta_0)\} = 0$ which can be used to solve for $\beta_0$. 

\item[\HandRight] If $0 < \theta_i < C_2$ then we have $\eta_i = 0$ and $\{ \eta_i + \rho_2 - y_i \left( \vect{\beta}^T \vect{h}(\vect{x_i}) + \beta_0 \right) \} = 0$ which can be used to solve for $\beta_0$. 

\item[\HandRight] Similar to C-SVM, for stability purposes we can average the estimate of $\beta_0$ over all points where $0 < \alpha_i < C_1$ and $0 < \theta_i < C_2$.

\item[\HandRight] We can calculate $\xi_i$ for those points for which $\alpha_i = C_1$ using $ \xi_i = \rho_1 - y_i \left( \vect{\beta}^T \vect{h}(\vect{x_i}) + \beta_0 \right)$. 

\item[\HandRight] Similarly, if $\theta_i = C_2$ then $\eta_i =  y_i \left( \vect{\beta}^T \vect{h}(\vect{x_i}) + \beta_0 \right) - \rho_2$.

\end{itemize}
\end{itemize}

\section{Toy data}
In order to illustrate the differences between C-SVM and B-SVM we generated artificial data in 2 dimensions as follows: 
\begin{itemize}

\item[\PencilRight] Class $1$ consisted of 5 bivariate Normal clusters centered at $(0,0)$, $(\frac{1}{\sqrt{2}}, \frac{1}{\sqrt{2}})$, $(\frac{-1}{\sqrt{2}}, \frac{1}{\sqrt{2}})$, $(\frac{-1}{\sqrt{2}}, \frac{-1}{\sqrt{2}})$ and $(\frac{1}{\sqrt{2}}, \frac{-1}{\sqrt{2}})$ and covariance $\sigma^2_1 \matr{I}_2$ with $\sigma_1 = 0.2$.

\item[\PencilRight] Class $-1$ consisted of 4 bivariate Normal clusters centered at $(1,0)$, $(0,1)$, $(-1,0)$ and $(0,-1)$ with covariacne $\sigma^2_2 \matr{I}_2$ with $\sigma_2 = 0.2$. 

\end{itemize}

A radial basis function (RBF) kernel was chosen for computations. For the RBF kernel, the elements of $\matr{K}$ are given by:
\begin{equation}\label{eq5}
K(\vect{x_i}, \vect{x_j}) = K_{ij} = \mbox{exp} \left\{ -\gamma \left( \vect{x_i} - \vect{x_j} \right)^T \left( \vect{x_i} - \vect{x_j} \right) \right\}
\end{equation}

Our parameter settings were as follows:

\begin{itemize}

\item[\PencilRight] For both C-SVM and B-SVM we used the same kernel parameter $\gamma = 1$. 

\item[\PencilRight] For C-SVM was  used $C = 10$.

\item[\PencilRight]  For B-SVM we chose $\rho_1 = 1$ and $C_1 = 10$ (same as $C$ for C-SVM). Thus the parameters of the common penalty term $C_1 \sum_{i = 1}^n [\rho_1 - y_i(\vect{\beta}^T \vect{h}(\vect{x_i}) + \beta_0)]_{+}$ are chosen to be identical for C-SVM and B-SVM. 

\item[\PencilRight] The parameters of the second penalty term for B-SVM were chosen as $C_2 = 100$ and $\rho_2 = 1.5$. Thus B-SVM will encourage $g(\vect{x})$ to lie in the interval $[\rho_1, \rho_2] = [1, 1.5]$.

\end{itemize}

\begin{figure}[htbp]
\begin{center}
\includegraphics[width = 6.0in] {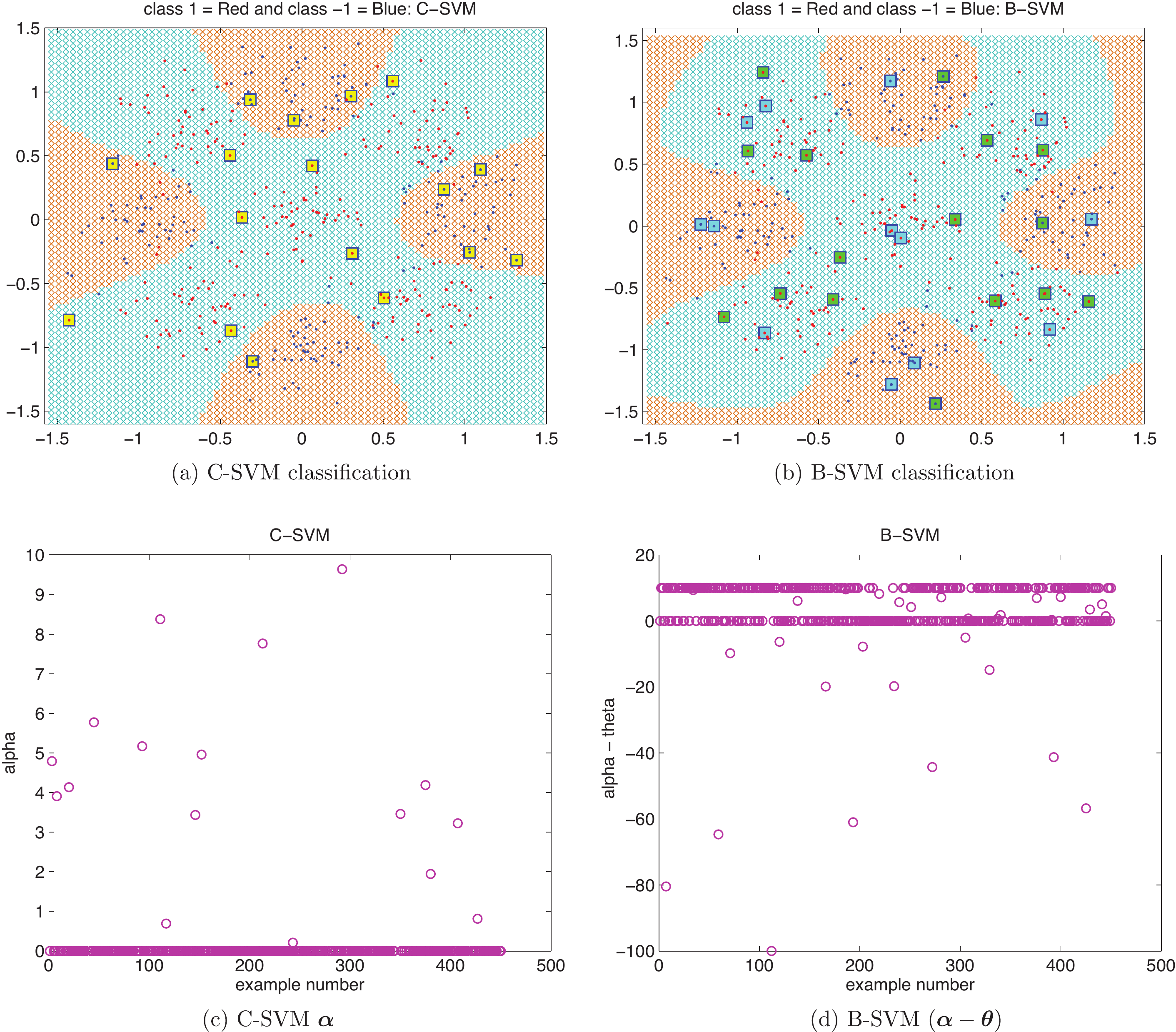}
\caption{Figure shows classification obtained for example data using (a) C-SVM and (b) B-SVM. \textcolor{red}{Red} and \textcolor{blue}{Blue} points (.) correspond to class $+1$ and $-1$ respectively. \textcolor{cyan}{Cyan} and \textcolor{orange}{Orange} x-marks (x) show the C-SVM and B-SVM decision rules evaluated at various points. Class $1$ membership is indicated in \textcolor{cyan}{Cyan} and class $-1$ membership is indicated in \textcolor{orange}{Orange}. The \textcolor{Yellow}{yellow} squares in (a) correspond to support points for which $0 < \alpha_i < C$. The \textcolor{cyan}{cyan} squares in (b) correspond to support points for which $0 < \theta_i < C_2$ and the \textcolor{ForestGreen}{green} squares correspond to support points for which $0 < \alpha_i < C_1$. The sparsity of solution is controlled by $\vect{\alpha}$ in the case of C-SVM and $(\vect{\alpha} - \vect{\theta})$ in the case of B-SVM (c) Shows $\alpha_i$ values for C-SVM. (d) Shows $(\alpha_i - \theta_i)$ values for B-SVM.  }
\label{figure2}
\end{center}
\end{figure}

\begin{figure}[htbp]
\begin{center}
\includegraphics[width = 6.0in] {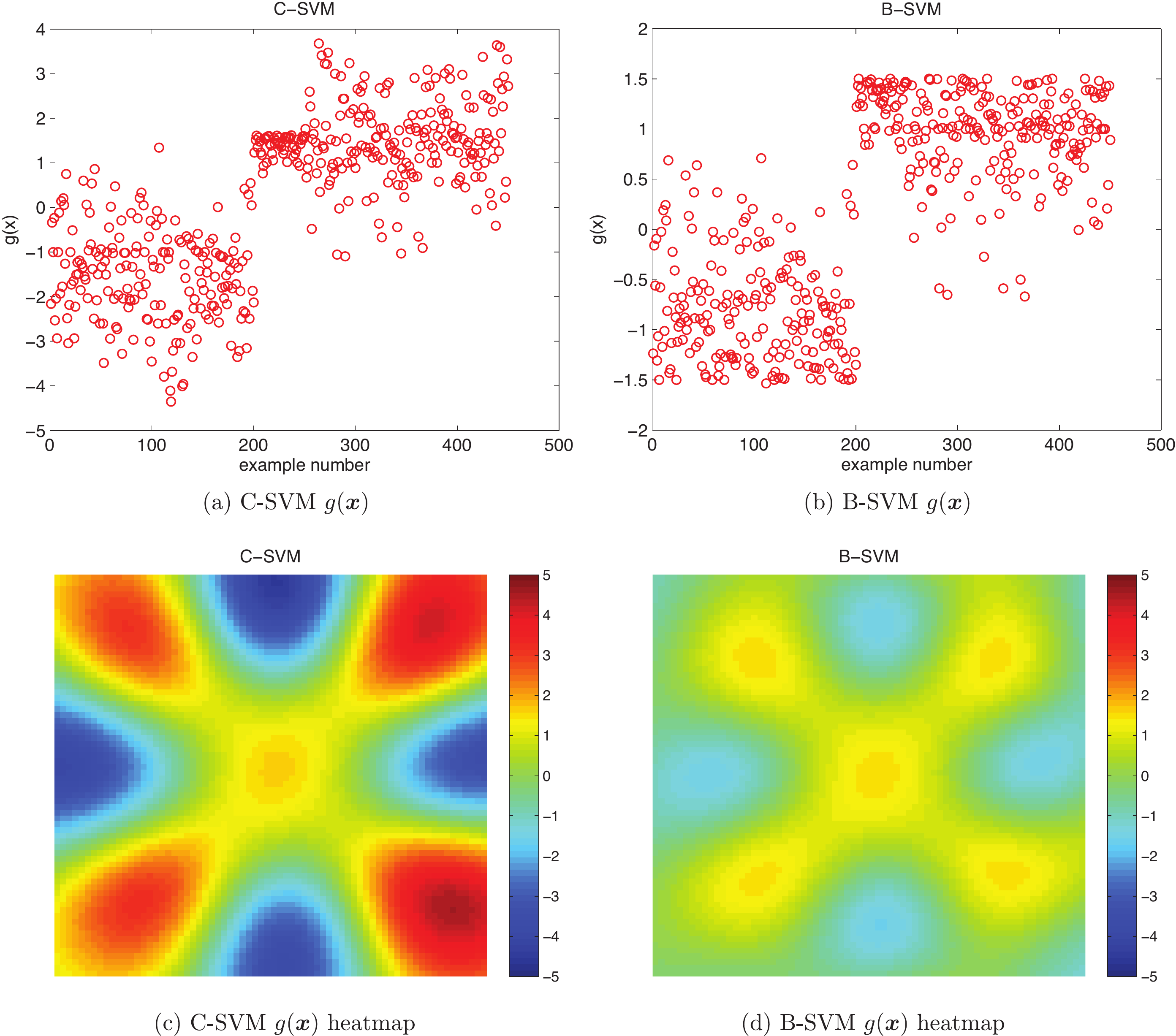}
\caption{Figure shows decision rule $g(\vect{x})$ for C-SVM (a) and B-SVM (b). Note that in B-SVM the second penalty term $C_2 \sum_{i = 1}^n [y_i(\vect{\beta}^T \vect{h}(\vect{x_i}) + \beta_0) - \rho_2]_{+}$ results in most of the $g(\vect{x})$ values in the interval $[\rho_1, \rho_2] = [1, 1.5]$. (c) Heat map of the decision rule $g(\vect{x})$ for C-SVM (d) Heat map of the decision rule $g(\vect{x})$ for B-SVM. In C-SVM the values of decision rule $g(\vect{x})$ are unbalanced in Class $1$. The central cluster located at $(0,0)$ in Class $1$ gets much smaller $g(\vect{x})$ values in C-SVM than the rest of the Class $1$. In B-SVM however, all clusters in Class $1$ including the one centered at $(0,0)$ get similar $g(\vect{x})$ values. This is a result of the second penalty term in the B-SVM objective function.}
\label{figure3}
\end{center}
\end{figure}

\begin{figure}[htbp]
\begin{center}
\includegraphics[width = 6.0in] {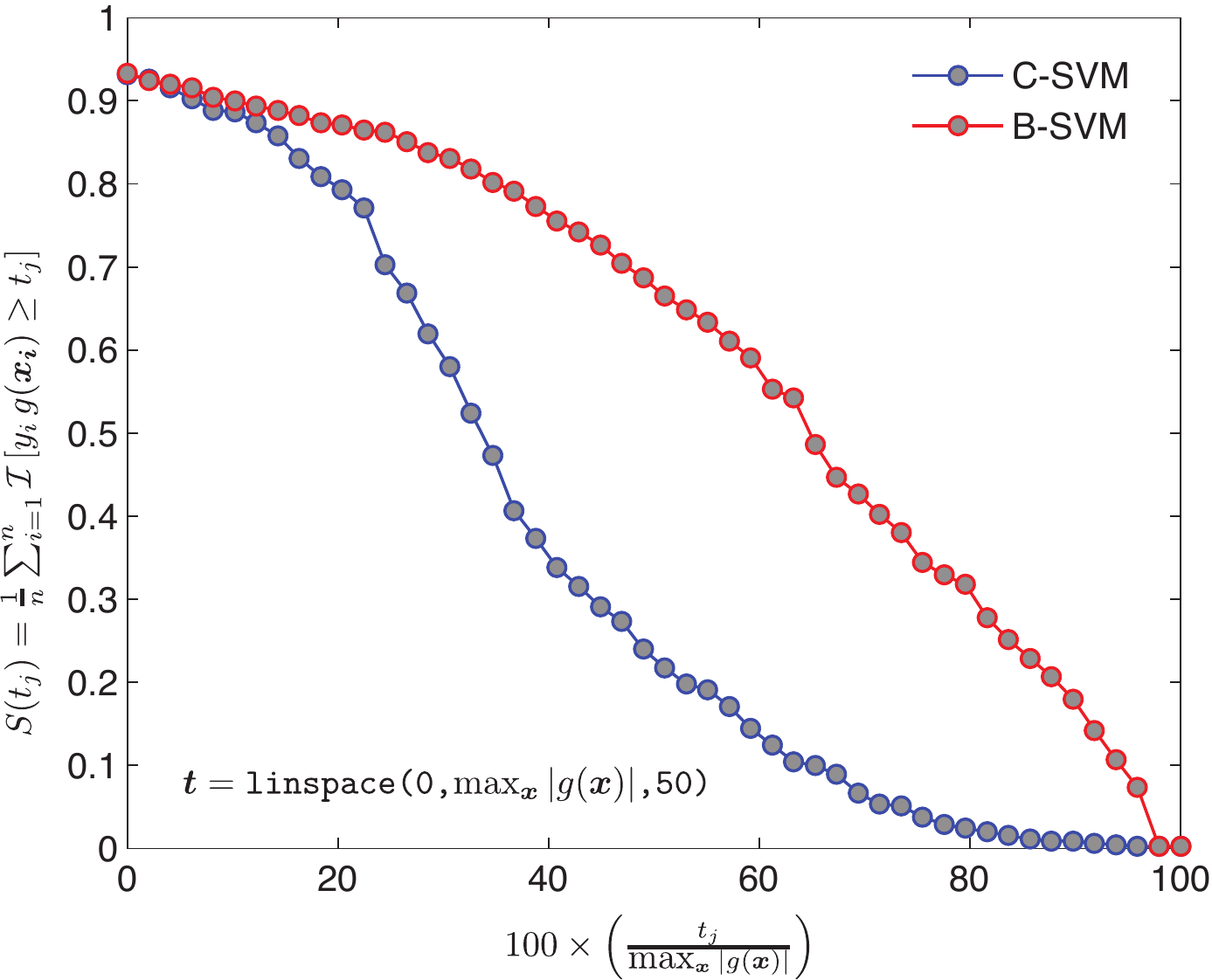}
\caption{Figure shows the fraction of points classified correctly by both C-SVM (\textcolor{blue}{blue curve}) and B-SVM (\textcolor{red}{red curve}) as a function of the decision rule threshold. The $x$-axis shows the decision rule threshold as a percentage of the maximum absolute value of the decision function $g(\vect{x})$ over all training points. The $y$-axis shows the overall classification accuracy or sensitivity of C-SVM and B-SVM.}
\label{figure4}
\end{center}
\end{figure}

Both C-SVM and B-SVM were fitted to the toy data described above. The following differences in the two solutions are noteworthy:

\subsection{$\alpha$-SVs and $\theta$-SVs}
The B-SVM dual problem \ref{8a} contains two variables $\vect{\alpha}$ and $\vect{\theta}$. Both $\alpha_i$ and $\theta_i$ are positive and satisfy the bound constraints given in \ref{8a}. Therefore, similar to C-SVM, we define 2 types of support vectors (SVs) in B-SVM:
\begin{itemize}
\item[\PencilRight] Points $i$ for which $\theta_i > 0$ are called the $\theta$-SVs  $\,\,\,$ \HandLeft $\,$ new SVs that arise in B-SVM
\item[\PencilRight] Points $i$ for which $\alpha_i > 0$ are called the $\alpha$-SVs $\,\,\,$ \HandLeft $\,$ standard C-SVM like SVs
\end{itemize}
Figures \ref{figure2}(a) and \ref{figure2}(b) show the C-SVM and B-SVM induced classification respectively for this example problem. Figure \ref{figure2}(b) shows $\alpha$-SVs for which $0 < \alpha_i < C_1$ and $\theta$-SVs for which $0 < \theta_i < C_2$. It is clear from \ref{eq4d} that the sparsity of a B-SVM decision rule depends on the quantities $(\alpha_i - \theta_i)$. Figures \ref{figure2}(c) and \ref{figure2}(d) show a plot of $\alpha_i$ for C-SVM and $(\alpha_i - \theta_i)$ for B-SVM respectively.

\subsection{Bounded decision rule}
Figures \ref{figure3}(a) and \ref{figure3}(b) show the decision rule values $g(\vect{x})$ over all training points for C-SVM and B-SVM. Recall that C-SVM does not enforce an upper limit on $g(\vect{x})$ whereas B-SVM attempts to encourage $g(\vect{x})$ to lie in $[\rho_1, \rho_2]$. It can be seen in Figure \ref{figure3}(b) that B-SVM was successful in limiting the absolute value of $g(\vect{x})$ to be $ < \rho_2 = 1.5$ with $C_2 = 100$. Figures \ref{figure3}(c) and \ref{figure3}(d) show a heat map of the decision rule for C-SVM and B-SVM respectively evaluated over a 2-D grid containing the training points. It can be seen that:
\begin{itemize}
\item[\PencilRight] The C-SVM decision rule values are \textit{unbalanced} in class $+1$ as the central cluster in class $+1$ gets lower $g(\vect{x})$ values compared to other clusters in class $+1$.
\item[\PencilRight] The decision rule values are \textit{balanced} in class $+1$ for B-SVM.
\end{itemize}

\subsection{Sensitivity curve}
We calculate the quantity: 
\begin{equation}\label{sen_eq1}
S(t) = \frac{1}{n} \sum_{i=1}^n \mathcal{I}\left[ y_i \, g(\vect{x_i}) \ge t \right]
\end{equation}
which is simply the fraction of correctly classified points (or sensitivity) using decision rule $g(\vect{x})$ at threshold $t$. To illustrate the variation in sensitivity of C-SVM and B-SVM decision rules:
\begin{itemize}
\item[\PencilRight] For both C-SVM and B-SVM, we divide the range of $g(\vect{x})$ into $50$ equally spaced points as follows (in MATLAB notation):
\begin{equation}\label{sen_eq2}
\vect{t} = \verb+linspace(0,+\mbox{max}_{\vect{x}} \, | g(\vect{x}) | \verb+,50)+
\end{equation}

\item[\PencilRight] Then we plot $100 \times \left( \frac{ t_j }{ \mbox{max}_{\vect{x}} \, | g(\vect{x}) |} \right)$ versus $S(t_j)$. 
\end{itemize}
Figure \ref{figure4} shows this sensitivity curve. It can be seen that for the same percentage threshold on the decision rule range:
\begin{itemize}
\item[\PencilRight] B-SVM has higher classification accuracy (or is more sensitive) than C-SVM. 
\item[\PencilRight] This effect is because of the balanced nature of decision rule values in B-SVM compared to C-SVM (see Figure \ref{figure3}(c) and \ref{figure3}(d)).
\end{itemize}

\section{Discussion and conclusions}
In this work, we considered the binary classification problem when the feature vectors in individual classes have finite co-variance. We showed that B-SVM is a natural generalization to C-SVM in this situation. It turns out that the B-SVM dual maximization problem \ref{8a} retains the concavity property of its C-SVM counterpart and C-SVM turns out to be a special case of B-SVM when $C_2 = 0$.
Two types of SVs arise in B-SVM, the $\alpha$-SVs which are similar to the standard SVs in C-SVM and $\theta$-SVs which arise due to the novel B-SVM objective function penalty \ref{2}. The B-SVM decision rule is more balanced than the C-SVM decision rule since it assigns $g(\vect{x})$ values that are comparable in magnitude to different sub-classes (or clusters) of class $+1$ and class $-1$. In addition, B-SVM retains higher classification accuracy compared to C-SVM as the decision rule threshold is varied from $0$ to $\mbox{max}_{\vect{x}} \, |g(\vect{x})|$. For a training set of size $n$, B-SVM results in a dual optimization problem of size $2n$ compared to a C-SVM dual problem of size $n$. Hence it is computationally more expensive to solve a B-SVM problem. 

In summary, B-SVM can be used to enforce balanced decision rules in binary classification. It is anticipated that the C-SVM leave one out error bounds for the \textit{bias free} case given in \cite{Jaakkola1999} will continue to hold in a similar form for \textit{bias free} B-SVM as well. 

\newpage
\bibliographystyle{plainnat}
\bibliography{pendse_bsvm_bib.bib}

\end{document}